\documentclass{article} %
\usepackage{iclr2027_conference,times}

\usepackage{amsmath,amsfonts,bm}

\def\1{\bm{1}}

\DeclareMathAlphabet{\mathsfit}{\encodingdefault}{\sfdefault}{m}{sl}
\SetMathAlphabet{\mathsfit}{bold}{\encodingdefault}{\sfdefault}{bx}{n}

\usepackage{hyperref}
\usepackage{url}
\usepackage[most]{tcolorbox}
\usepackage{mathtools}
\usepackage{multirow}
\usepackage{booktabs}
\usepackage{amsthm}
\usepackage{graphicx}
\usepackage{subcaption}
\newtheorem{lemma}{Lemma}

\title{A Critical look at Power Sampling when Test-Time Sharpening LLMs}

\author{Abiel Malkani Talwar, Christopher Nemeth, Henry Moss \\
Lancaster University \\
Lancaster, United Kingdom \\
\texttt{\{a.m.talwar,c.nemeth,henry.moss\}@lancaster.ac.uk}
}

\iclrfinalcopy %
\begin{document}

\maketitle
\thispagestyle{fancy}
\lhead{Preprint.}

\begin{abstract}
Reinforcement learning with verifiable rewards (RLVR) has substantially improved LLM reasoning, yet growing evidence suggests it mainly performs \emph{distribution sharpening}: concentrating probability mass on reasoning paths the base model already produces. Recent work therefore seeks RLVR's gains without needing reward signals, by sampling from the power distribution $\pi_\alpha \propto \pi_{\text{base}}^\alpha$ via Markov chain Monte Carlo or sequential Monte Carlo. We show that these samplers fail to converge to $\pi_\alpha$ under realistic compute budgets, and that their gains are largely recreated via simple Best-of-$N$ with appropriate grading . This suggests they succeed by selecting high-probability sequences, not by approximating $\pi_\alpha$, raising a natural question: is the power distribution itself a worthwhile target? We test this hypothesis directly by fine-tuning the base model to match $\pi_\alpha$, sidestepping sampler convergence issues entirely. Despite approximating the target far more faithfully than existing samplers, this strategy yields gains that are noisy and specific to certain models and sampling tasks. Together, our results suggest that the value of power sampling lies not in faithfully targeting $\pi_\alpha$, but in the simpler goal of extracting sequences the base model already deems likely.
\end{abstract}

\section{Introduction}
\label{sec:introduction}

Reinforcement learning with verifiable rewards (RLVR) has become a cornerstone of large language model (LLM) post-training \citep{shao2024deepseekmathpushinglimitsmathematical, Guo_2025, hu2025openreasonerzeroopensourceapproach, lambert2025tulu3pushingfrontiers}. It has proven especially effective in domains with checkable answers, such as mathematics, code generation, and formal reasoning, where it elicits multi-step reasoning that is difficult to instil through supervised fine-tuning alone \citep{chen2021evaluatinglargelanguagemodels, Li_2022, hou2024largelanguagemodelssoftware, zimmer2025bourbakiselfgeneratedgoalconditionedmdps}.

A growing body of work, however, questions whether RLVR unlocks genuinely new reasoning behaviour or merely amplifies what the base model can already do. \citet{yue2025doesreinforcementlearningreally} find that its accuracy gains come at the cost of diversity, as revealed by pass@$k$, while \citet{he2025rewardingunlikelyliftinggrpo} show that GRPO exhibits a degenerate rank bias, funnelling additional probability mass toward sequences the base model already ranks highly. These findings raise a natural question: can the gains attributed to RLVR be recovered simply by sampling more cleverly from the base model's own high-probability regions, with no verifier or reward signal at all?

A growing body of work claims the answer is yes: simply sample from the \emph{power distribution} $\pi_\alpha \propto \pi_{\text{base}}^\alpha$ with $\alpha > 1$.\citet{karan2025reasoningsamplingbasemodel} proposed Markov chain Monte Carlo (MCMC) sampling, and subsequent work turned to sequential Monte Carlo (SMC) for efficiency \citep{azizi2026powersmclowlatencysequencelevelpower, wang2026cache}. These methods are motivated by a compelling theory --- that faithfully approximating $\pi_\alpha$ should yield substantial reasoning improvements.

In this work, we critically examine this popular claim. We analyze why existing samplers fail, propose POWER-RLOO to directly target $\pi_\alpha$ with high fidelity, and test whether faithful approximation actually drives reasoning gains. Our contribtuions and conclusions are as follows:
\begin{itemize}
\itemsep0.5em 
    \item \textbf{We show that recent SMC schemes fail to converge} to $\pi_\alpha$ (Section~\ref{sec:pitfalls}), instead behaving as likelihood-guided search with little sensitivity to $\alpha$ (Figure~\ref{fig:log_prob_comparision}).
    
    \item \textbf{We find that SMC gains are largely recreated by a simple Best-of-$N$ variant} that periodically resamples the highest-likelihood partial sequences, suggesting that these samplers succeed by selecting high-probability sequences rather than by approximating $\pi_\alpha$.

    \item  \textbf{We introduce Power-RLOO as a fine-tuning scheme that faithfully targets $\pi_\alpha$}. Critical to efficiency, we exploit the connection between power distribution and low-temperature sampling, allowing our adapter to learn only the non-local lookahead correction (Section~\ref{sec:power-rloo}). 

    \item \textbf{We find that faithful approximation is insufficient.} Despite approximating the target far more effectively than previous work (Figure~\ref{fig:log_prob_comparision}), we find that explicitly targeting $\pi_\alpha$ yields accuracy gains beyond low-temperature decoding only when the base model is already reasonably well-aligned with the task, contradicting claims that approximation quality drives improvements.
    \item \textbf{We show sharpening elicits latent capacity, not just selection.} Despite the inherently sharpening nature of the power distribution, modest fine-tuning actually increases pass@k, suggesting it reveals existing reasoning capacity rather than merely narrowing the output distribution—a capacity simpler selection methods can then also access. (Section~\ref{sec:analysis}).

\end{itemize}

\begin{figure}[t]
    \centering
    \includegraphics[width=0.98\linewidth]{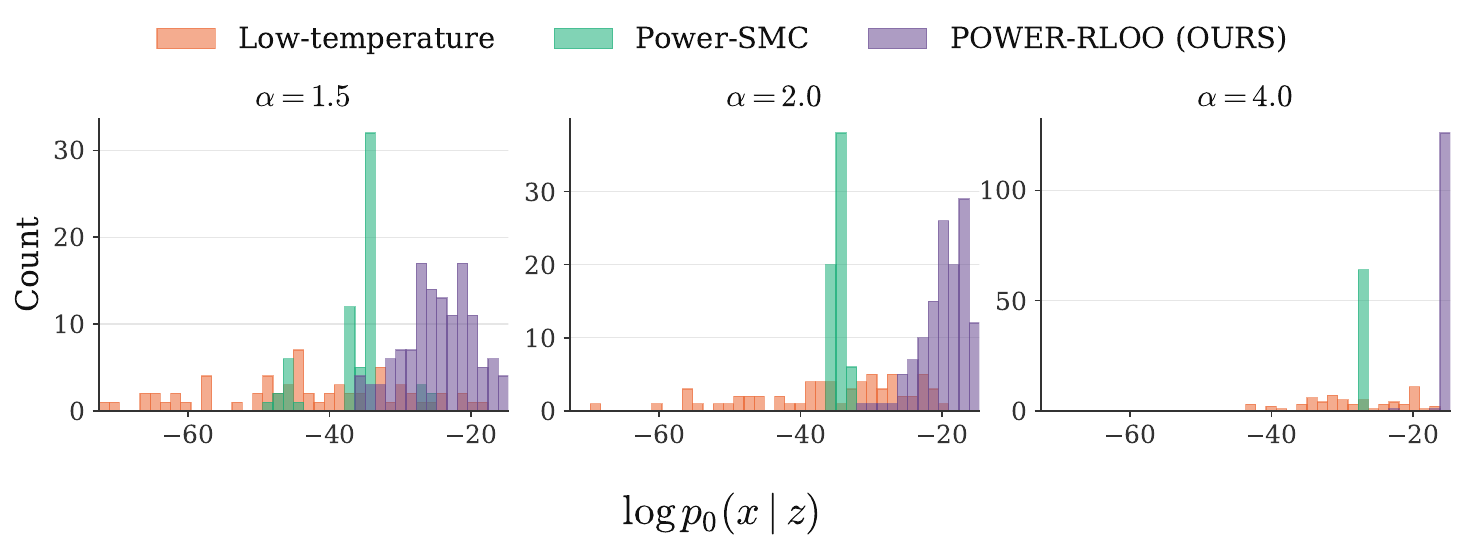}
    \caption{Base-model log-likelihoods of samples drawn from low-temperature decoding ($\tau = 1/\alpha$), Power-SMC, and our \textsc{Power-RLOO} on Qwen2.5-7B-Instruct (MATH500, Question 1), for $\alpha \in \{1.5, 2.0, 4.0\}$. Power-SMC's samples collapse into a narrow, near-identical likelihood band that is largely insensitive to $\alpha$, consistent with particle degeneracy. The narrow band indicates fewer than 10 effective samples. Power-RLOO, in contrast, shifts its mass toward higher likelihood as $\alpha$ grows, matching the behavior expected under $\pi_\alpha$.}
    \label{fig:log_prob_comparision}
\end{figure}

\section{Background and Related Work}
\label{sec:background}

Conditioned on a prompt $z$, a large language model (LLM) defines a probability distribution over the set of finite sequences via a series of conditional probabilities as follows. Let $\mathcal{V}$ denote the vocabulary (token set), and let $x = (x_1, x_2, \dots, x_T) \in \mathcal{V}^T$ denote a generated sequence of length $T$, where $x_T = \texttt{<eos>}$ is a distinguished end-of-sequence token. The model defines a conditional distribution over the next token given the prompt and previously generated tokens:
\[
p_0(x_t \mid z, x_{<t}),
\]
where $x_{<t} = (x_1, \dots, x_{t-1})$ denotes the (possibly empty) prefix of tokens generated so far. The probability of an entire sequence $x$ given the prompt $z$ is then obtained via the chain rule of probability:
\begin{equation}
p_0(x \mid z) = \prod_{t=1}^{T} p_0(x_t \mid z, x_{<t}).
\label{eq:base-model}    
\end{equation}

This autoregressive factorization allows the model to generate sequences token-by-token, sampling each $x_t$ from $p_\theta(\cdot \mid z, x_{<t})$ conditioned on all previously generated tokens, until the \texttt{<eos>} token is produced or a maximum length is reached.

\subsection{Sharpening LLMs at Test-time}

\textbf{Low Temperature Sampling}. A common simple strategy for biasing generation towards higher-likelihood continuations is low temperature sampling, in which each conditional is rescaled by a parameter $\tau > 0$:
\begin{equation}
    p_0^\tau(x_t \mid z, x_{<t}) \propto p_0(x_t \mid z, x_{<t})^{1/\tau}.
    \label{eq:temp_conditional}
\end{equation}
Taking $\tau < 1$ sharpens the distribution over the next token, concentrating probability mass on the model's most likely continuations at each step, with $\tau \to 0$ recovering greedy (most-likely) decoding and $\tau = 1$ recovering the original conditional $p_0(x_t \mid z, x_{<t})$. Sampling a full sequence under temperature $\tau$ then proceeds as above by applying the chain rule to these tempered conditionals, adding negligible additional cost over standard LLM sampling.
However, as the rescaling in \eqref{eq:temp_conditional} is applied independently at each timestep, this sequence-level distribution is only ever tempering the model's \emph{immediate} next-token beliefs --- it has no mechanism for taking into account how a given choice of $y_t$ shapes the likelihood of the tokens that follow it.

\textbf{Power Sampling}. To account for the likelihood of these ``future paths,'' \citet{karan2025reasoningsamplingbasemodel} instead propose sampling from the \emph{power distribution} which acts on the sequence as a whole:
\begin{equation}
    \pi_\alpha(x \mid z) = \frac{p_0(x \mid z)^\alpha}{Z_\alpha(z)} \quad \textrm{for} \quad Z_\alpha(z) = \sum_{x' \in \mathcal{X}} p_0(x' \mid z)^\alpha , \qquad \alpha > 1.
    \label{eq:power_distribution_w_const}
\end{equation}
Here, the normalizing constant $Z_\alpha(z)$, which sums over all possible sequences $x' \in \mathcal{X}$, is analytically intractable which means that direct simulation from \eqref{eq:power_distribution_w_const} is infeasible as we no longer have an autoregressive structure. Although strategies have been proposed via approximate sampling algorithms, based on MCMC \citep{karan2025reasoningsamplingbasemodel} and SMC \citep{ji2026scalablepowersamplingunlocking,azizi2026powersmclowlatencysequencelevelpower}. However, we will later demonstrate, in Section~\ref{sec:pitfalls}, that these samplers fail to accurately sample the power distribution \eqref{eq:power_distribution_w_const} under realistic compute budgets.

\subsection{Fine-tuning LLMs }

In Section \ref{sec:power-rloo} we will propose an alternative approach for sampling an LLM's power distribution --- fine-tuning the LLM to directly target the power distribution.%
We now provide requisite overview of LLM fine-tuning methods.

\textbf{Fine-tuning via reinforcement learning.} RL fine-tuning algorithms such as GRPO \citep{shao2024deepseekmathpushinglimitsmathematical} assign a reward $r(x)$ to each sample $x$ drawn from the model and update the policy to shift probability mass toward high-reward samples. When this reward can be computed automatically by checking the output against a ground-truth answer or test suite, rather than by a learned reward model or human judgment, the setting is known as reinforcement learning with verifiable rewards (RLVR). Fine-tuning is achieved by optimizing a RL objective of the form
\begin{equation}
    J_{\text{KL-RL}}(\theta) = \mathbb{E}_{x \sim \pi_\theta}\left[ r(x) \right] - \beta \, D_{\text{KL}}\!\left( \pi_\theta \,\|\, \pi_{\text{base}} \right),
    \label{eq:kl_rl_objective}
\end{equation}
where $\beta$ controls the regularizing penalty for diverging from the base model. In maximizing \eqref{eq:kl_rl_objective}, the fine-tuned model assigns more probability to high-reward trajectories. More precisely, as derived by \cite{korbak2022rlklpenaltiesbetter}, the fine-tuning update is a form of variational inference with target posterior
\begin{equation}
    \pi^*_{\text{KL-RL}}(x) = \frac{1}{Z} \, \pi_{\text{base}}(x) \exp\!\left( r(x) / \beta \right).
    \label{eq:kl_rl_posterior}
\end{equation}

In practice, to implement this we use LoRA \citep{hu2021loralowrankadaptationlarge}, a parameter-efficient fine-tuning paradigm that exploits the rank deficiency of LLMs, representing weight updates as low-rank matrices while keeping the base model $\pi_0$ frozen.

\paragraph{Self-Improvement.} Our setting of fine-tuning towards the power distribution fits into the field of \emph{self-improvement} \citep{huang2022largelanguagemodelsselfimprove, wang2023selfinstructaligninglanguagemodels, yuan2025selfrewardinglanguagemodels}, in which intrinsic signals from within the model are used in place of external rewards or data. For example, \citet{huang2024selfimprovementlanguagemodelssharpening} explore using the model's own log-probability as a reward signal,
\begin{equation}
    r_{\text{self}}(x \mid z) := \log \pi_{\text{base}}(x \mid z),
    \label{eq:self_reward_def}
\end{equation}
both at inference time and within an RLHF loop, to target a distribution over responses that maximizes the self-reward;
\begin{equation}
    \hat{\pi}(z) \approx \operatorname*{arg\,max}_{x \in \mathcal{X}} \, r_{\text{self}}(x \mid z; \pi_{\text{base}}).
    \label{eq:self_improvement_sharpening}
\end{equation}
Notably, this target coincides with taking $\alpha \to \infty$ in the power distribution \eqref{eq:power_distribution_w_const}, recovering the mode of $\pi_{\text{base}}$. \citet{huang2024selfimprovementlanguagemodelssharpening} find that this reward signal consistently improves accuracy over greedy decoding across a range of benchmarks, lending empirical support to this choice of reward.

\textbf{(Un)importance of reward quality}. The underlying mechanics of RLVR  remain an active area of research, and a growing body of work has begun to challenge some of its assumptions. \citet{shao2026spuriousrewardsrethinkingtraining} show that mathematical reasoning can be improved even when trained on reward signals that are weakly, or even negatively, correlated with true correctness, while 1-shot RLVR \citep{wang2025reinforcementlearningreasoninglarge} shows that a single training example can recover most of the gains typically attributed to training on thousands. In a related vein, self-certainty and other internal-feedback methods \citep{zhao2026learningreasonexternalrewards} dispense with ground-truth answers and verifiers altogether, using the model's own confidence as the sole reward signal and still matching RLVR performance in-domain, suggesting that the specific correctness signal driving RLVR may matter far less than is commonly assumed.

\section{The Pitfalls of Existing Sharpening Schemes}
\label{sec:pitfalls}

We consider three inference-time schemes recently proposed for sampling
from $\pi_\alpha$: Metropolis--Hastings
\citep{karan2025reasoningsamplingbasemodel}, rollout estimation of the power
conditional \citep{ji2026scalablepowersamplingunlocking}, and sequential Monte
Carlo \citep{azizi2026powersmclowlatencysequencelevelpower}. Although each is theoretically exact in
the infinite limit, we show that
at practical budgets each scheme collapses to the same behavior: keeping the
highest-likelihood candidates and discarding the rest. The samplers therefore act
as a likelihood-guided search, largely insensitive to $\alpha$, and their gains
cannot be attributed to the power distribution.

\paragraph{MCMC.} \citet{karan2025reasoningsamplingbasemodel} propose sampling
from $\pi_\alpha$ approximately using a Metropolis--Hastings (MH) scheme.
Candidate blocks, or subsequences, of tokens are iteratively proposed and
accepted according to the MH acceptance ratio \eqref{eq:acceptance_ratio}. The
authors show that, compared to sampling from the base model
\eqref{eq:base-model}, their strategy steers generation towards
higher-probability sequences. However, the size of the underlying sequence
space means that MCMC schemes of this kind suffer from slow mixing and
substantial compute overhead per generated sample. Furthermore, proposing at the level
of blocks rather than tokens means that the log-likelihoods being compared
differ by tens of nats, which reduces the sensitivity to $\alpha$.

Let $p_{\mathrm{prop}}$ be the proposal that simulates new tokens $x'$,
conditional on the existing tokens $x$, which are assumed to be distributed
according to $\pi_\alpha$. The new tokens $x'$ are accepted with probability
\begin{equation}
  Acc(x', x) \;=\; \min\Bigg(
    1,\;
    \underbrace{\frac{\pi_\alpha(x')}{\pi_\alpha(x)}}_{\rule{0pt}{2.2ex}A}
    \cdot
    \underbrace{\frac{p_{\mathrm{prop}}(x \mid x')}{p_{\mathrm{prop}}(x' \mid x)}}_{\rule{0pt}{2.2ex}B}\Bigg).
  \label{eq:acceptance_ratio}
\end{equation}
Figure~\ref{fig:mcmc_diag}(a--b) shows that by increasing $\alpha$, the term $A$ begins to dominate the acceptance ratio \eqref{eq:acceptance_ratio}. This means that a proposal is accepted whenever its log-probability is high, with the proposal term $B$ contributing little to the probability of acceptance. We compare 400 MCMC steps against a greedy hill climber, which corresponds with $B$ = 1 (red).  The log-probability traces in
Figure~\ref{fig:mcmc_diag}(c--d) show the same pattern: as $\alpha$ grows the
chains stagnate at high-density regions rather than exploring. Despite this,
\citet{karan2025reasoningsamplingbasemodel} demonstrate improved reasoning, suggesting
that the benefit comes not from convergence to the power distribution but from
the mode-seeking nature of their algorithm.
\begin{figure}[t]
    \centering
    \begin{subfigure}[b]{0.47\textwidth}
        \centering
        \includegraphics[width=\textwidth]{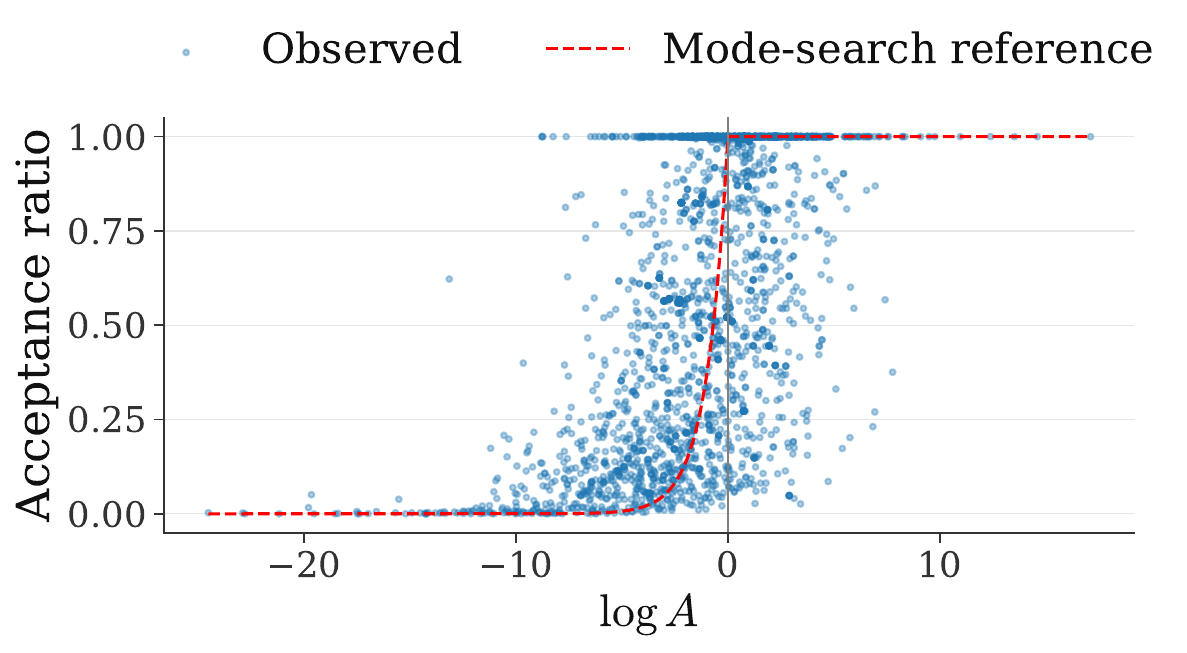}
        \caption{Acceptance Ratio vs. $\log A$ for $\alpha =2$}
        \label{fig:mcmc_acc_2}
    \end{subfigure}
    \begin{subfigure}[b]{0.47\textwidth}
        \centering
        \includegraphics[width=\textwidth]{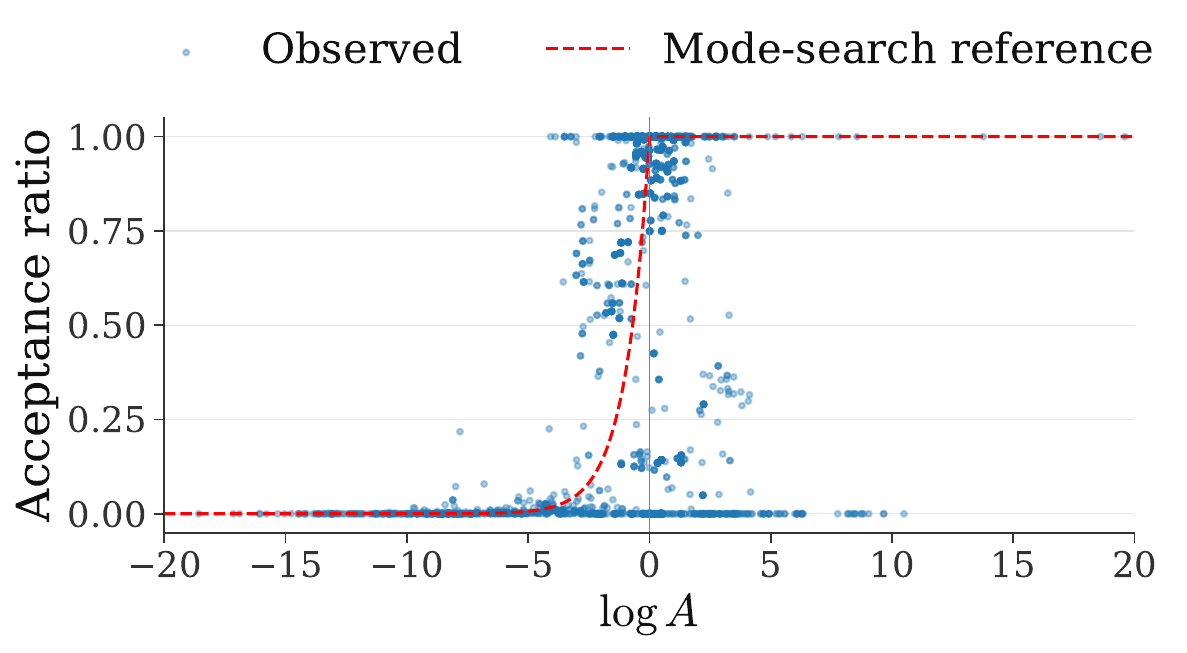}
        \caption{Acceptance Ratio vs. $\log A$ for $\alpha =4$}
        \label{fig:mcmc_acc_4}
    \end{subfigure}
    \vspace{1em}

    \begin{subfigure}[b]{0.47\textwidth}
        \centering
        \includegraphics[width=\textwidth]{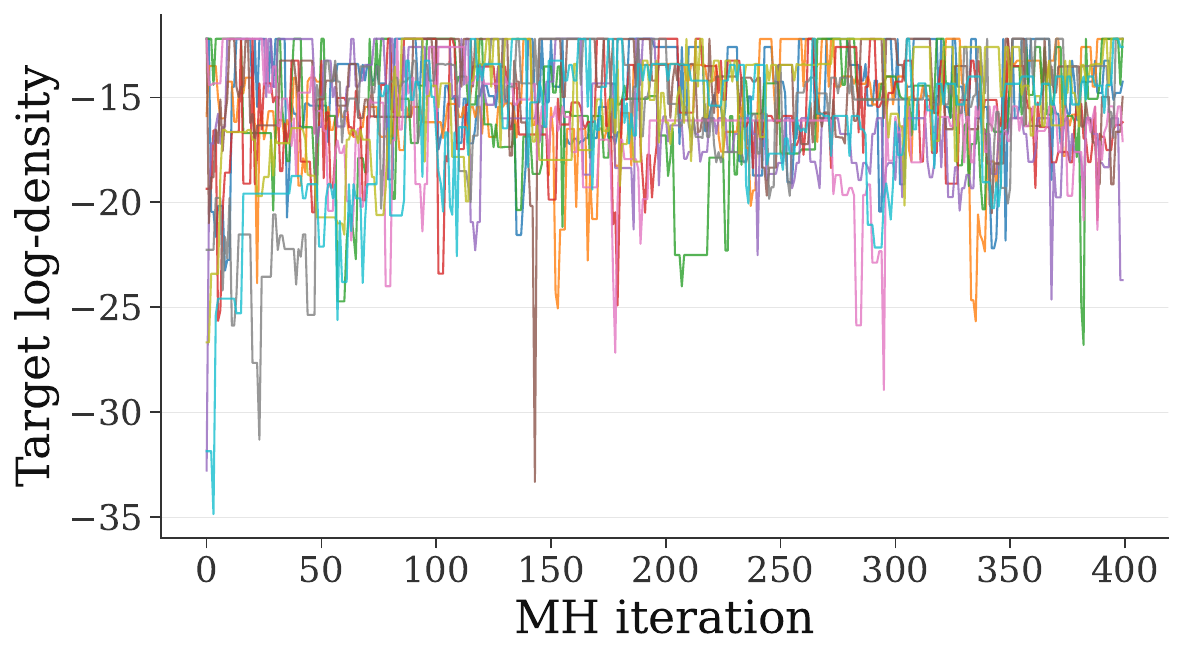}
        \caption{Trace plot for $\alpha =2$}
        \label{fig:mcmc_trace_2}
    \end{subfigure}
    \begin{subfigure}[b]{0.47\textwidth}
        \centering
        \includegraphics[width=\textwidth]{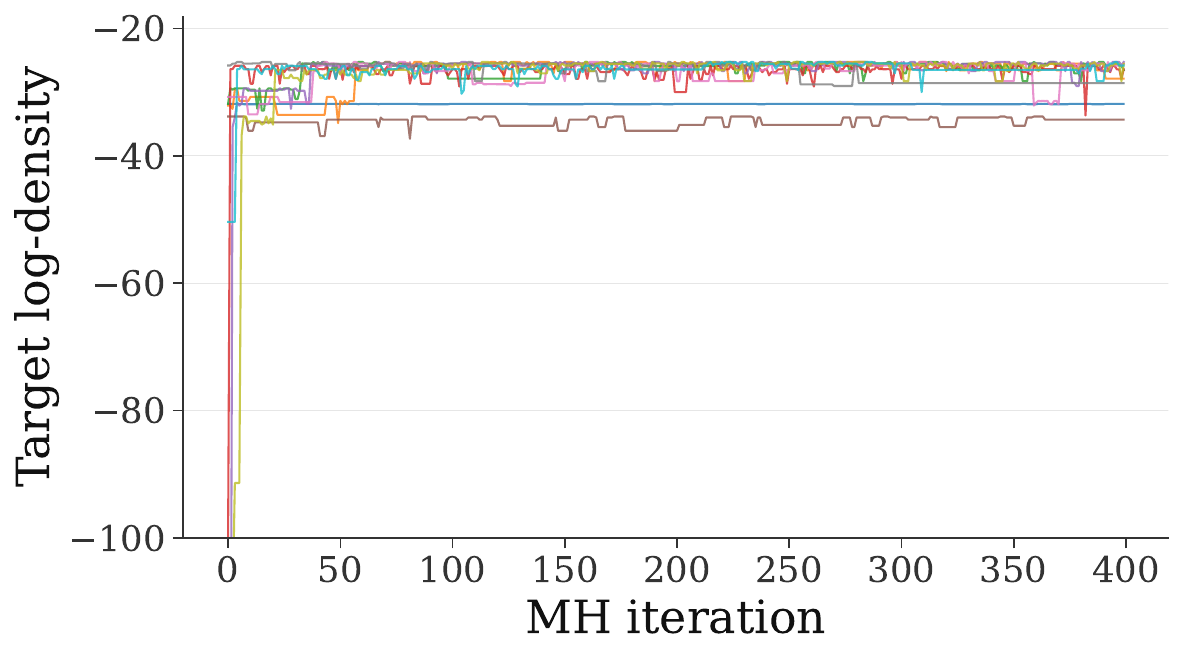}
        \caption{Trace plot for $\alpha =4$}
        \label{fig:mcmc_trace_4}
    \end{subfigure}

    \caption{Diagnostics for the MCMC sampler. As $\alpha$ increases, the chains progressively fail to explore: the trace plots (c--d) show increasing stagnation, and the acceptance behavior (a--b) reveals the sampler collapsing into a greedy hill-climber that accepts almost exclusively on the basis of higher likelihood.}
    \label{fig:mcmc_diag}
\end{figure}

\paragraph{Rollout estimation of the twist.}
\citet{ji2026scalablepowersamplingunlocking} observe that the conditional
distribution induced by $\pi_\alpha$ at each timestep is a tempered next-token
distribution reweighted by a correction factor $\zeta_t$,
\begin{equation}
    \pi_\alpha(x_t \mid z, x_{<t}) = \frac{p_0(x_t \mid z, x_{<t})^\alpha \, \zeta_t(x_t)}{\sum_{x_t' \in \mathcal{V}} p_0(x_t' \mid z, x_{<t})^\alpha \, \zeta_t(x_t')}; \: \textrm{for} \: 
    \zeta_t(x_t') = \sum_{x_{t+1:T}} p_0(x_{t+1:T} \mid z, x_{<t}, x_t')^\alpha,
    \label{eq:power_conditional}
\end{equation}
where $\mathcal{V}$ is the vocabulary. We refer to $\zeta_t$, the tempered partition function over all completions,  as the
\emph{twist}. \cite{ji2026scalablepowersamplingunlocking} estimate this intractable quantity using Monte Carlo with rollout blocks of horizon $H$ from the base model. However, each term in this average is the likelihood of a completion hundreds of tokens long, raised to the power $\alpha-1$, and such terms vary across rollouts by hundreds of nats. The average is therefore dominated by its largest term, so $\hat\zeta_t$ is effectively determined by the single best rollout.
The candidate weights are then a softmax over the log-likelihood, so for any $\alpha>1$ the selected block is approximately the one whose best rollout has the highest likelihood, motivating subsequent work that
avoids explicit twist estimation altogether.

\paragraph{Power-SMC.}
\citet{azizi2026powersmclowlatencysequencelevelpower} avoid estimating the twist altogether by maintaining a population of sequences each with importance weights. They run a standard SMC sampler on the \textit{prefix targets} $\gamma_t(x_{1:t}) = p(x_{1:t}\mid z)^\alpha$. A population of $N$ particles is advanced one token at a time under a prefix-only proposal $q_t(\cdot \mid z, x_{<t})$, each particle's weight is multiplied by the incremental importance weight $\omega_t = p_0(x_t\mid z,x_{<t})^\alpha / q_t(x_t\mid z, x_{<t})$, and the population is resampled whenever the effective sample size (ESS) falls below $\kappa N$.  However, the log-weight of a particle becomes a sum of $t$ non-positive terms whose magnitudes vary from prefix to prefix, so the variance of log-weights across the population grows
linearly in $t$ and the ESS collapses within a few tokens of every reset.

\begin{figure}[t]
    \centering
    \begin{subfigure}[b]{0.48\textwidth}
        \centering
        \includegraphics[width=\textwidth]{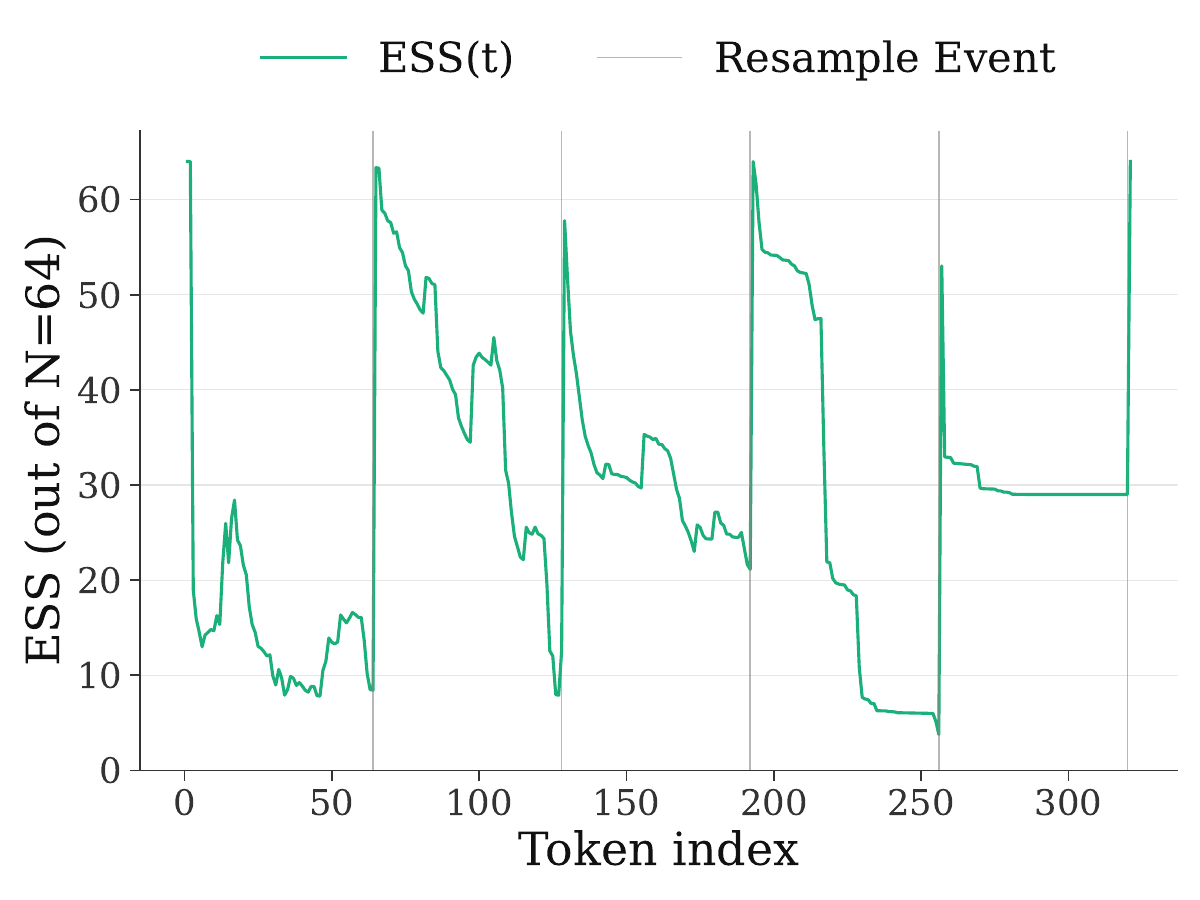}
        \caption{ESS when resampling}
        \label{fig:smc_ess}
    \end{subfigure}
    \hfill
    \begin{subfigure}[b]{0.48\textwidth}
        \centering
        \includegraphics[width=\textwidth]{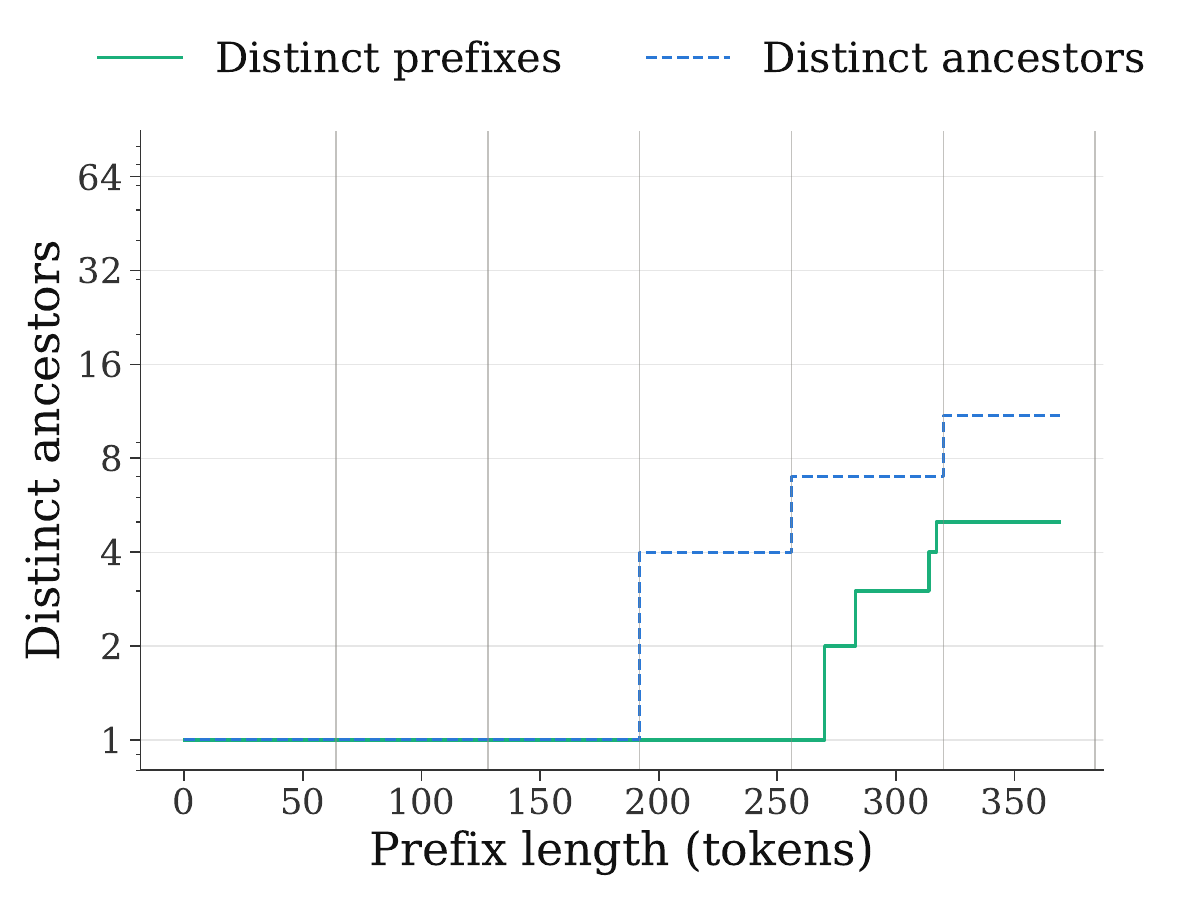}
        \caption{Distinct Ancestors}
        \label{fig:smc_ancestors}
    \end{subfigure}
    \caption{Ancestral diversity and effective sample size over the SMC trajectory. Resampling occurs whenever diversity drops too low, which, when compounded over the sequence, leaves the final population with few distinct ancestors and limited diversity.}
    \label{fig:smc_diag}
\end{figure}

Figure~\ref{fig:smc_diag} shows the ESS trajectory at $\alpha=4$ with $N=64$ and $\kappa=0.5$, where the ESS is checked every $64$ tokens. Between resampling events, the ESS decays from $64$ to $\approx 8$ within 64 tokens (panel~(a)), so that by the time of resampling the normalized weight is carried by a handful of particles. Panel~(b) traces the ancestry of the final population: by the end of generation the particles share only 6 distinct prefixes. Combined with a proposal that is already sharp, the population from which the final sample is drawn consists of a few near-identical sequences, as demonstrated by the narrow, $\alpha$-insensitive band of log-likelihoods in Figure~\ref{fig:log_prob_comparision}.

To investigate this degeneracy further, we propose running a simple baseline where we replace weighted
resampling by deterministic selection. Our \emph{Best-of-$N$ with Resampling}
baseline runs $N$ low-temperature samples in parallel and, every $R$ tokens,
keeps the $k$ prefixes with the highest base-model log-likelihood and
duplicates them to restore the population. Note that this simple selection rule ranks by cumulative log-likelihood whereas Power-SMC ranks by cumulative R\'enyi entropy of the traversed
prefixes.

\section{POWER-RLOO}
\label{sec:power-rloo}

We now propose Power-RLOO, a method to fine-tune an LLM to directly target $\pi_\alpha$, avoiding the need for SMC approximations. We establish that $\pi_\alpha$ is the optimum of a KL-regularized RL objective with log-likelihood rewards, and derive an efficient token-level gradient estimator. This advance will enable a direct analysis of whether approximating $\pi_\alpha$ improves reasoning. Throughout, $z$ denotes the prompt, $x = (x_1, \dots, x_T)$ a sampled response, $\pi_0 \equiv \pi_{\text{base}}$ the frozen base model, and $\pi_\theta$ the model being fine-tuned, initialized at $\pi_0$.

\textbf{Sharpening as KL-regularized RL.} Recall from \eqref{eq:kl_rl_posterior} that the KL-regularized fine-tuning objective in \eqref{eq:kl_rl_objective} is maximized by $\pi^*(x \mid z) \propto \pi_0(x \mid z) \exp\big(r(x)/\beta\big)$. Choosing the reward
\begin{equation}
    r(x) = \beta(\alpha - 1) \log \pi_0(x \mid z)
    \label{eq:power_reward_choice}
\end{equation}
gives $\pi^* \propto  \pi_0^{\,\alpha}$, so recovers the power distribution $\pi_\alpha$. Substituting this reward into \eqref{eq:kl_rl_objective} gives:
\begin{equation*}
    J_{\text{KL-RL}}(\theta)
    = \beta \, \mathbb{E}_{x \sim \pi_\theta(\cdot \mid z)}\big[ \alpha \log \pi_0(x \mid z) - \log \pi_\theta(x \mid z) \big]
    = -\beta \Big[ D_{\text{KL}}\big(\pi_\theta \,\|\, \pi_\alpha\big) - \log Z_\alpha(z) \Big],
\end{equation*}
where $Z_\alpha(z)$ is the normalizing constant. Since $\beta > 0$ only rescales the objective and $\log Z_\alpha(z)$ does not depend on $\theta$, maximizing $J_{\text{KL-RL}}$ is equivalent to minimizing the reverse KL divergence to $\pi_\alpha$. Neither $\beta$ nor the intractable $Z_\alpha$ enters the optimization.

\paragraph{RLOO Training Objective.} Factorizing both log-probabilities with the chain rule reveals a token-level reward structure, which we exploit through several variance-reduction techniques to derive a REINFORCE reward-to-go ($G_t$) estimator. Incorporating a leave-one-out baseline, $b_t$ then yields our final gradient estimator;

\begin{equation}
  \nabla_\theta \mathcal{L}
  = \mathbb{E}\!\left[ \sum_{t} \nabla_\theta \log \pi_\theta(x_t \mid z, x_{<t}) \cdot (G_t - b_t) \right],
  \label{eq:final-body}
\end{equation}
with full details and derivation provided in Appendix~\ref{app:trainingobj}.

\paragraph{Efficient fine-tuning by warm start.} The power conditional \eqref{eq:power_conditional} factorizes as $\pi_\alpha(x_t \mid z, x_{<t}) \propto \pi_0(x_t \mid z, x_{<t})^{\alpha}\, \zeta_t(x_t)$: a locally tempered term, times the lookahead twist $\zeta_t$. We therefore sample and score the fine-tuned model at temperature $1/\alpha$ throughout, giving
\begin{equation}
    \pi_\theta(x_t \mid z, x_{<t}) \propto \exp\!\big(\alpha\,[\ell_0(x_t) + \Delta\ell_{\mathrm{LoRA}}(x_t)]\big),
    \label{eq:lowtemp_policy}
\end{equation}
where $\ell_0$ are the base model's logits and $\Delta\ell_{\mathrm{LoRA}}$ is the adapter's contribution. At initialization ($\Delta \ell_{\text{LoRA}} = 0$), $\pi_\theta(x_t \mid z, x_{<t}) \propto \pi_0(x_t \mid z, x_{<t})^{\alpha}$, and matching it to $\pi_\alpha$ shows that the adapter need only learn
\begin{equation}
    \Delta\ell_{\mathrm{LoRA}}(v) = \tfrac{1}{\alpha} \log \zeta_t(v) + \underbrace{c(z, x_{<t})}_{\text{constant}},
    \label{eq:adapter_target}
\end{equation}
i.e the hard to approximate lookahead \emph{twist} term. Appendix~\ref{app:renyi} shows that, under this parametrization, the cost-to-go reduces to the cumulative R\'enyi entropy of future continuations at initialization.  Appendix~\ref{app:warm_start} demonstrates that our warm start requires substantially less training, as we begin closer to the target distribution.%

\begin{tcolorbox}[
    enhanced,
    colback=blue!3,
    colframe=blue!50!black,
    coltitle=white,
    fonttitle=\bfseries,
    title=Benefits of \textsc{Power-RLOO},
    arc=2mm,
    boxrule=0.8pt,
    left=6pt, right=6pt, top=6pt, bottom=6pt,
    lefttitle=8pt, righttitle=8pt,
    toptitle=4pt, bottomtitle=4pt
]
\begin{enumerate}
    \item \textbf{Measurable fidelity.} Because $\pi_\theta$ is an explicit model, we can evaluate its likelihoods and track the loss $\mathcal{L}(\theta) = D_{\text{KL}}(\pi_\theta \,\|\, \pi_\alpha) - \log Z_\alpha$ throughout training. This directly measures how closely we approach the target, rather than inferring it from an approximate sampler.
    \item \textbf{Clean sensitivity analysis.} This lets us isolate the effect of $\alpha$, both on the learned distribution itself and on downstream sample quality.
    \item \textbf{Simplicity.} Training is straightforward to implement and requires none of the engineering machinery (block proposals, resampling, cache management) required by MCMC- and SMC-based approaches.
\end{enumerate}
\end{tcolorbox}

\section{Experiments}
\label{sec:experiments}

We can now, for the first time, concretely assess the value of targeting the power distribution. We deploy \textsc{Power-RLOO}, \textsc{Power-SMC}, and additional heuristics across a range of models and two benchmarks. Full code will be released upon acceptance.

\subsection{Setup}

\paragraph{Models and Datasets.} We evaluate on Qwen2.5-7B (Base and Instruct) and Phi-3.5-mini-Instruct. We use MATH500~\citep{lightman2023letsverifystepstep}, a set of 500 competition-level mathematics problems spanning algebra, geometry, number theory, and combinatorics, to assess multi-step mathematical reasoning. We additionally evaluate on the GPQA Diamond dataset ~\citep{rein2023gpqagraduatelevelgoogleproofqa}, a set of graduate-level, multiple-choice questions in biology, physics, and chemistry written by domain experts to be resistant to simple web lookup.

\paragraph{Generation Config and Hyperparameters.} To target the power distribution as faithfully as possible, we disable top-$p$ and top-$k$ sampling and remove any repetition penalties, ensuring that the distribution we sample from (and compare against) remains unaltered. We follow existing work, and focus on targeting $\alpha=4$. For fine-tuning, we use LoRA (with rank 32) and train in stages across $\alpha \in \{1.5, 2, 3, 4\}$, training for 200 steps at each value of $\alpha$ with a learning rate of 1e-4. Due to compute restraints we limit generation length to 1564. All training was done on a single Nvidia H200 GPU. Comprehensive hyperparameters are detailed in Appendix \ref{app:implementation}.

\paragraph{Baselines.} As motivated in Section~\ref{sec:pitfalls}, we implement a simple resampling scheme in which, after every block of 64 tokens, we select the top-$k$ samples by log-probability and duplicate them; we refer to this as Best-of-$N$ w/ Resampling. We report values with $k=4$ and $N = 16$. We additionally compare against standard Best-of-$N$ \citep{huang2024selfimprovementlanguagemodelssharpening}, where we generate $N$ low temperature samples ($\tau = 0.25$) and select the highest-scoring sequence according to either total or mean log-probability:
\begin{align*}
  \hat{x}_{\mathrm{sum}}  &\;=\; \arg\max_{x} \; \sum_{t} \log p_\theta(x_t \mid z, x_{<t}), \label{eq:best_of_n_sum} \\
  \hat{x}_{\mathrm{mean}} &\;=\; \arg\max_{x} \; \frac{1}{|x|} \sum_{t} \log p_\theta(x_t \mid z, x_{<t}). 
\end{align*}

\subsection{Results and Analysis}
\label{sec:analysis}

We report accuracy across all three models and both benchmarks in Table~\ref{tab:main_results}, observing the following notable patterns in line with the intuition developed by Figure~\ref{fig:log_prob_comparision}.

\begin{table}[t]
\centering
\begin{tabular}{llcc}
\toprule
\textbf{Model} & \textbf{Method} & \textbf{MATH500 (\%)} & \textbf{GPQA (\%)} \\
\midrule
\multirow{6}{*}{Qwen2.5-7B-Base}
    & Baseline decoding                    & 39.0   & 28.6   \\
    & Low-temperature decoding ($\tau = 1/\alpha$) & 63.1   & 31.4   \\
    & Power-SMC                            &  \textbf{70.6}   & 31.9   \\
    & Best-of-$N$-mean                          & 67.2   & \textbf{32.2}   \\
    & Best-of-$N$-sum                          & 62.6   & 32.2   \\
    
    & Best-of-$N$ w/ Resampling            &  \textbf{70.0} & 29.7   \\
    & Power-RLOO  (ours)                        & N/A   & N/A   \\
\midrule
\multirow{6}{*}{Qwen2.5-7B-Instruct}
    & Baseline decoding                    & 73.1   & 35.9   \\
    & Low-temperature decoding ($\tau = 1/\alpha$) & 75.4   & 35.1   \\
    & Power-SMC                            & 76.9   & 35.8   \\
    
    & Best-of-$N$-mean                          &  77.0   & 36.5   \\
    & Best-of-$N$-sum                          & 75.9   & 36.5  \\
    & Best-of-$N$ w/ Resampling            & 76.9   & 32.6   \\
    & Power-RLOO      (ours)                     & \textbf{78.4}   & \textbf{41.0}   \\
\midrule
\multirow{6}{*}{Phi-3.5-mini-instruct}
    & Baseline decoding                    & 38.9   & 29.1   \\
    & Low-temperature decoding ($\tau = 1/\alpha$) & 41.7   & \textbf{29.4}   \\
    & Power-SMC                            & \textbf{46.6}   & 28.0   \\
    & Best-of-$N$-mean                          &  43.8   & 28.7   \\
    & Best-of-$N$-sum                          & 43.8    & 28.6   \\
    & Best-of-$N$ w/ Resampling            & 45.6   & 26.5   \\
    & Power-RLOO  (ours)                         & 43.6   & 29.2   \\
\bottomrule
\end{tabular}
\caption{Pass@1 accuracy on MATH500 and GPQA at $\alpha=4$. \textsc{Power-RLOO} and Power-SMC yield only modest, model-dependent gains over simpler mode-seeking baselines such as Best-of-N. N/A denotes a lack of convergence
due to an undesirable length degeneracy pathology arising from more precisely targeting the power distribution (see Section \ref{sec:analysis})}
\label{tab:main_results}
\end{table}

\paragraph{Sharpening works best on top of a strong prior.} We find that a faithful representation of the power distribution is only beneficial when the underlying base model is already reasonably well-aligned with the task at hand. This is most evident on Qwen2.5-7B-Instruct, where the more accurate sharpening achieved by \textsc{Power-RLOO} leads to the strongest results of any method on both benchmarks (78.4\% on MATH500 and 41.0\% on GPQA). By contrast, on the weaker Phi-3.5-mini-Instruct model, \textsc{Power-RLOO} trails Power-SMC on MATH500 (43.6\% vs.\ 46.6\%), suggesting that sharpening only helps if there is a meaningful signal in the base distribution to sharpen toward.

\paragraph{Standard Best-of-$N$ frequently matches or exceeds Power-SMC.} For GPQA, Best-of-$N$-mean outperforms Power-SMC for every model (e.g.\ Qwen-Base: 32.2\% vs.\ 31.9\%; Qwen-Instruct: 36.5\% vs.\ 35.8\%; Phi: 28.7\% vs.\ 28.0\%), and
  even the more accurate sharpening achieved by our \textsc{Power-RLOO}. On MATH500, Best-of-$N$-mean matches Power-SMC for Qwen-Instruct (77.0\% vs.\ 76.9\%). 
  These results reinforce our core message: much of the benefit attributed to SMC-based sharpening is achievable through simple, training-free selection.

\paragraph{Benefits over low temperature sampling are not guaranteed.} For MATH500, the additional gain from low temperature sampling to Power-RLOO or SMC is sometimes sizeable (e.g.\ Qwen-Base: $63.1\% \to 70.6\%$ with Power-SMC),  however, on GPQA the gap between low-temperature decoding and the more sophisticated methods closes.  Therefore, the gains of sharpening vary significantly by task structure.

\textbf{Faithfully targeting $\pi_\alpha$ leads to length degeneracy}. Shorter trajectories carry higher per-sequence probability and are consequently favored by the power distribution. On Qwen-Base, POWER-RLOO collapsed toward degenerate, overly short completions (denoted N/A in Table \ref{tab:main_results}). Notably, SMC approximations and Best-of-N avoid this problem, precisely because they fail to converge to $\pi_\alpha$. This reveals an important trade-off: faithful approximation of $\pi_\alpha$ exposes pathological behavior that inference-time samplers sidestep through convergence failure. Maintaining reasonable output length under sharpening requires a strong prior over length itself -- one such as that instilled via chat-template tuning in instruct models.

\paragraph{Weak sharpening uncovers latent knowledge.} Sharpening naturally reduces the typical set of sampled trajectories, and a faithful approximation of the power distribution should reflect this. Indeed, as shown in Figure~\ref{fig:pass@k}, POWER-RLOO's pass@k decreases with increasing $\alpha$; however, mild sharpening actually improves this metric at large $k$, suggesting it surfaces latent knowledge already present in the model but rarely sampled at $\alpha=1$. This reveals a practically valuable property of sharpening: while aggressive mode-seeking optimizes raw accuracy, mild sharpening exposes existing reasoning capacity—arguably a more useful outcome than further narrowing the distribution. Power-SMC's poor pass@k stems from particle degeneracy; we report pass@k over re-weighted particles from a single run rather than k independent populations, which would be prohibitively expensive.

\begin{figure}[htbp]
    \centering
    \includegraphics[width=0.8\textwidth]{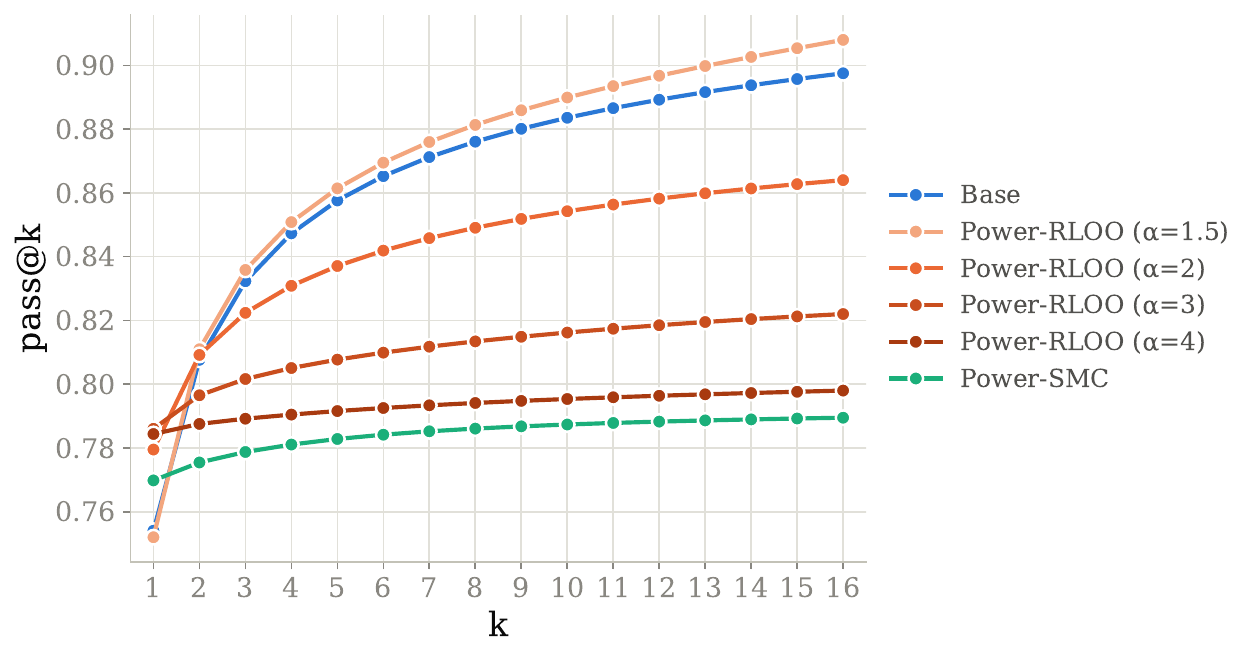}
    \caption{Pass@$k$ accuracy for \textsc{Power-RLOO} and Power-SMC ($\alpha=4$) on Qwen2.5-7B-Instruct (MATH500). As expected, pass@$k$ decreases with $\alpha$; however, mild sharpening ($\alpha=1.5$) actually improves pass@$k$, suggesting it elicits reasoning paths that exist in the base model but are rarely sampled at $\alpha=1$}
    \label{fig:pass@k}
\end{figure}

\section{Conclusion}
We have critically examined recently proposed MCMC and SMC schemes for targeting an LLM's power distribution, showing that they fail to converge and instead behave as likelihood-guided search. We then introduce \textsc{Power-RLOO}, a principled fine-tuning algorithm that directly targets $\pi_\alpha$ and allows, for the first time, an precise investigation into the common claim that sharpening drives reasoning improvement.

\textbf{We found that faithful approximation of $\pi_\alpha$ yields only modest, model-dependent gains, and sometimes undesirable phenomenon, like length degeneracy}. Best-of-N matches or exceeds sharpening without targeting $\pi_\alpha$, demonstrating that improvements stem from selection, not approximation quality. The field should reconsider power sampling as a general path to reasoning improvement. %

\newpage

\subsection*{AI use statement}

In this work, we did not use generative AI tools for any of the tasks requiring disclosure (generating synthetic datasets; developing theoretical models or conceptual frameworks; formulating mathematical claims; providing critical ingredients for proving mathematical claims; assisting in the writing of proofs; proposing or refining hypotheses; designing or providing feedback on research methodology or experiments; implementing methods; assisting with translation; cleaning and reformatting datasets; supporting qualitative and thematic data analysis; or interpreting results). Additionally, we used generative AI tools for creating/modifying scientific figures (plotting help) and editing the research paper to improve readability (proofreading). We have reviewed all AI-assisted work. All AI-suggested plots were checked against the underlying data by the authors for accuracy before inclusion, and all AI-suggested text edits were manually reviewed against the original meaning to ensure no content or claims were altered. We take responsibility for the final content of this work, including text, claims or artifacts produced with the aid of generative AI.

\subsection*{Ethics statement}

This work focuses on improving the data efficiency of large language model training, and does not involve human subjects, new data collection, or the release of a new dataset. We do not anticipate significant risks of harm arising directly from our method. To the extent that our approach reduces the amount of data and compute required to train performant models, it may have a modest positive environmental impact by lowering the computational cost of training. As with any advance that improves the capability or accessibility of LLMs, we acknowledge the general dual-use consideration that more efficient training could also lower the barrier to developing models for harmful purposes; we believe this risk is not unique to our method and is outweighed by the broader benefits of more efficient, less resource-intensive model development. We are not aware of any conflicts of interest or sponsorship arrangements relevant to this work.

\subsection*{Reproducibility statement}

Experimental settings are provided in Section \ref{sec:experiments}, and full code will be published upon acceptance. 

\bibliography{iclr2027_conference}
\bibliographystyle{iclr2027_conference}
\newpage

\appendix
\section{Deriving the Training Objective}
\label{app:trainingobj}
We begin from the Kl informed loss
\begin{equation}
    J_{\text{KL-RL}}(\theta)
    = \beta \, \mathbb{E}_{x \sim \pi_\theta(\cdot \mid z)}\big[ \alpha \log \pi_0(x \mid z) - \log \pi_\theta(x \mid z) \big]
    = -\beta \Big[ D_{\text{KL}}\big(\pi_\theta \,\|\, \pi_\alpha\big) - \log Z_\alpha(z) \Big]. 
\end{equation}

Factorizing both log-probabilities with the chain rule, we obtain the training loss
\begin{equation}
    \mathcal{L}(\theta) \;=\; \mathbb{E}_{x \sim \pi_\theta(\cdot \mid z)}\!\left[ \sum_{t=1}^{T} f_t \right],
    \qquad
    f_t \;\coloneqq\; \log \pi_\theta(x_t \mid z, x_{<t}) \;-\; \alpha \log \pi_0(x_t \mid z, x_{<t}),
    \label{eq:jtheta}
\end{equation}
which we minimize. Because $\mathcal{L}(\theta) = D_{\text{KL}}(\pi_\theta \,\|\, \pi_\alpha) - \log Z_\alpha(z)$, its unique minimizer is $\pi_\theta = \pi_\alpha$.

We now derive \textsc{Power-RLOO}'s training objective using the following identity.

\begin{lemma}[Zero-mean conditional score]
\label{lem:score}
Let $Y$ be any random variable that is conditionally independent of $x_{t}$ given the prefix $x_{<t}$. Then
\begin{equation}
  \mathbb{E}_{x \sim \pi_\theta(\cdot \mid z)}\!\left[\, Y \, \nabla_{\theta} \log \pi_\theta(x_t \mid z, x_{<t}) \,\right] \;=\; 0 .
\end{equation}
\end{lemma}

\begin{proof}
Condition on the prefix $y_{<t}$ and apply the tower rule:
\begin{equation*}
  \mathbb{E}\!\left[ Y \, \nabla_{\theta} \log \pi_\theta(x_t \mid z, x_{<t}) \right]
  = \mathbb{E}_{x_{<t}}\!\left[\, \mathbb{E}\!\left[Y \mid x_{<t}\right] \cdot
      \mathbb{E}_{x_{t} \mid x_{<t}}\!\left[ \nabla_{\theta} \log \pi_\theta(x_{t} \mid z, x_{<t}) \right] \right],
\end{equation*}
where the factorization of the inner expectation follows from the conditional independence of $Y$ and $y_t$ given $y_{<t}$. The inner score term has zero mean:
\begin{align*}
  \sum_{v \in \mathcal{V}} \pi_\theta(v \mid z, x_{<t}) \, \nabla_{\theta} \log \pi_\theta(v \mid z, x_{<t})
  &= \sum_{v \in \mathcal{V}} \nabla_{\theta} \pi_\theta(v \mid z, x_{<t}) \\
  &= \nabla_{\theta} \sum_{v \in \mathcal{V}} \pi_\theta(v \mid z, x_{<t})
  = \nabla_{\theta} 1 = 0,
\end{align*}
where $\mathcal{V}$ denotes the vocabulary, so that summing over $v$ ranges over all possible next tokens.
\end{proof}

Differentiating Equation~\ref{eq:jtheta} via the product rule gives

\begin{equation}
  \nabla_{\theta} \mathcal{L}
  = \underbrace{\mathbb{E}\!\left[ \Big(\textstyle\sum_{s} f_{s}\Big)\, \nabla_{\theta} \log \pi_\theta(x \mid z) \right]}_{\text{distributional}}
  + \underbrace{\mathbb{E}\!\left[ \textstyle\sum_{t} \nabla_{\theta} f_{t} \right]}_{\text{pathwise}} .
  \label{eq:productrule}
\end{equation}

\paragraph{Step 1: the pathwise term vanishes.} Since $\alpha \log \pi_0(x_t \mid z, x_{<t})$ does not depend on $\theta$, we have $\nabla_{\theta} f_{t} = \nabla_{\theta} \log \pi_\theta(x_t \mid z, x_{<t})$, so Lemma 1 (with $Y=1$) annihilates each summand. This cancellation holds \emph{in expectation}, not per sample: dropping the pathwise term discards a mean-zero source of gradient noise, rather than a source of bias.

\paragraph{Step 2: the past drops out (reward-to-go).} Expanding $\nabla_{\theta} \log \pi_\theta(x \mid z) = \sum_{t} \nabla_{\theta} \log \pi_\theta(x_t \mid z, x_{<t})$ in the distributional term leaves a double sum $\sum_{t} \sum_{s} \mathbb{E}\big[ f_{s} \, \nabla_{\theta} \log \pi_\theta(x_t \mid z, x_{<t}) \big]$. For any pair with $s < t$, the quantity $f_{s}$ is a function only of $x_{\le s} \subseteq x_{<t}$, and is therefore conditionally independent of $x_{t}$ given $x_{<t}$; Lemma~\ref{lem:score} (with $Y = f_{s}$) kills each such term. Only the terms with $s \ge t$ survive, leaving

\begin{equation}
  \nabla_{\theta} \mathcal{L}
  = \mathbb{E}\!\left[ \sum_{t} \nabla_{\theta} \log \pi_\theta(x_t \mid z, x_{<t}) \cdot G_{t} \right],
  \qquad
  G_{t} \;\coloneqq\; \sum_{s \ge t} f_{s} .
  \label{eq:rtg}
\end{equation}

\paragraph{Step 3: subtract a baseline.} Lemma~\ref{lem:score} with $Y = b_{t}$, for any $b_{t}$ conditionally independent of $x_{t}$ given $y_{<t}$, gives the final estimator:
\begin{equation}
  \boxed{\;
  \nabla_{\theta} \mathcal{L}
  = \mathbb{E}\!\left[ \sum_{t} \nabla_{\theta} \log \pi_\theta(y_t \mid z, x_{<t}) \cdot \big(G_{t} - b_{t}\big) \right],
  \qquad A_{t} \coloneqq G_{t} - b_{t}. \;}
  \label{eq:final}
\end{equation}

We use the leave-one-out batch mean at each position,
\begin{equation}
  b_{i,t} = \frac{1}{B-1}\sum_{j \ne i} G_{j,t},
  \label{eq:rloo}
\end{equation}

In practice, one does not differentiate Equation~\ref{eq:final} by hand. Instead, we form the surrogate scalar
\begin{equation}
  \widehat{\mathcal{L}}(\theta)
  = \frac{1}{B}\sum_{i=1}^{B}\sum_{t=1}^{T} \mathrm{sg}\!\big(\hat{A}_{i,t}\big)\, \log \pi_\theta(x_{i,t} \mid z_i, x_{i,<t}),
  \label{eq:surrogate}
\end{equation}
where $\mathrm{sg}(\cdot)$ denotes the stop-gradient (\texttt{detach}) operator, and call \texttt{backward} on it directly.

\section{Connection to R\'enyi entropy}
\label{app:renyi}

Our low-temperature parameterization also gives the training signal a clean interpretation. Write $Z_t = \sum_v \pi_0(v \mid z, x_{<t})^{\alpha}$ and let $H_\alpha(p) = \frac{1}{1-\alpha} \log \sum_v p(v)^{\alpha}$ denote the R\'enyi entropy of order $\alpha$. At initialization, $\pi_\theta = \pi_0^{\alpha} / Z_t$, and the per-token cost in Equation~\eqref{eq:jtheta} reduces to
\begin{equation}
    f_t = \log \pi_\theta(x_t \mid z, x_{<t}) - \alpha \log \pi_0(x_t \mid z, x_{<t}) = -\log Z_t
        = (\alpha - 1)\, H_\alpha\!\big(\pi_0(\cdot \mid z, x_{<t})\big).
    \label{eq:renyi_cost}
\end{equation}
Notably, this cost does not depend on the sampled token $x_t$, only on the entropy of the distribution it was drawn from. The cost-to-go is therefore the cumulative R\'enyi entropy along the rest of the rollout,
\begin{equation}
    G_t = (\alpha - 1) \sum_{s \ge t} H_\alpha\!\big(\pi_0(\cdot \mid z, x_{<s})\big).
    \label{eq:renyi_rtg}
\end{equation}
Since the $s = t$ term is fixed by the prefix, the influence assigned to $x_t$ comes entirely from the entropy of its continuations. This observation mirrors the role of the twist in Equation~\eqref{eq:adapter_target}. Because we minimize $\mathcal{L}$, tokens whose continuations are less uncertain than the leave-one-out baseline have their logits raised, and those leading into more uncertain continuations have them lowered.

\section{Low Temperature Sampling in Power-RLOO}

In this section we explore the utility of always using low-temperature sampling when training with Power-RLOO. Appendix~\ref{app:renyi} gives a clean mathematical interpretation of the rewards under low-temperature sampling; here we show that low-temperature sampling also yields smoother and faster convergence to the target. 

Because Power-RLOO maintains a valid distribution throughout training, we can directly observe training dynamics via estimates of the KL divergence. As shown in Figure~\ref{fig:warm_start} (global step 0), the low-temperature sampling distribution starts closer to the power distribution in terms of KL. This is expected: Equation~\eqref{eq:power_conditional} shows that the two are inextricably linked via the twist term learned by the LoRA adapter, Equation~\eqref{eq:adapter_target}. We also observe that the warm start transfers across jumps in $\alpha$, avoiding the sharp KL jump encountered otherwise. This may be interpreted as evidence that the twist learned for a smaller $\alpha$ is a good starting approximation for the twist required at a larger $\alpha$.

\label{app:warm_start}
\begin{figure}[h]
    \centering
    \includegraphics[width=0.9\linewidth]{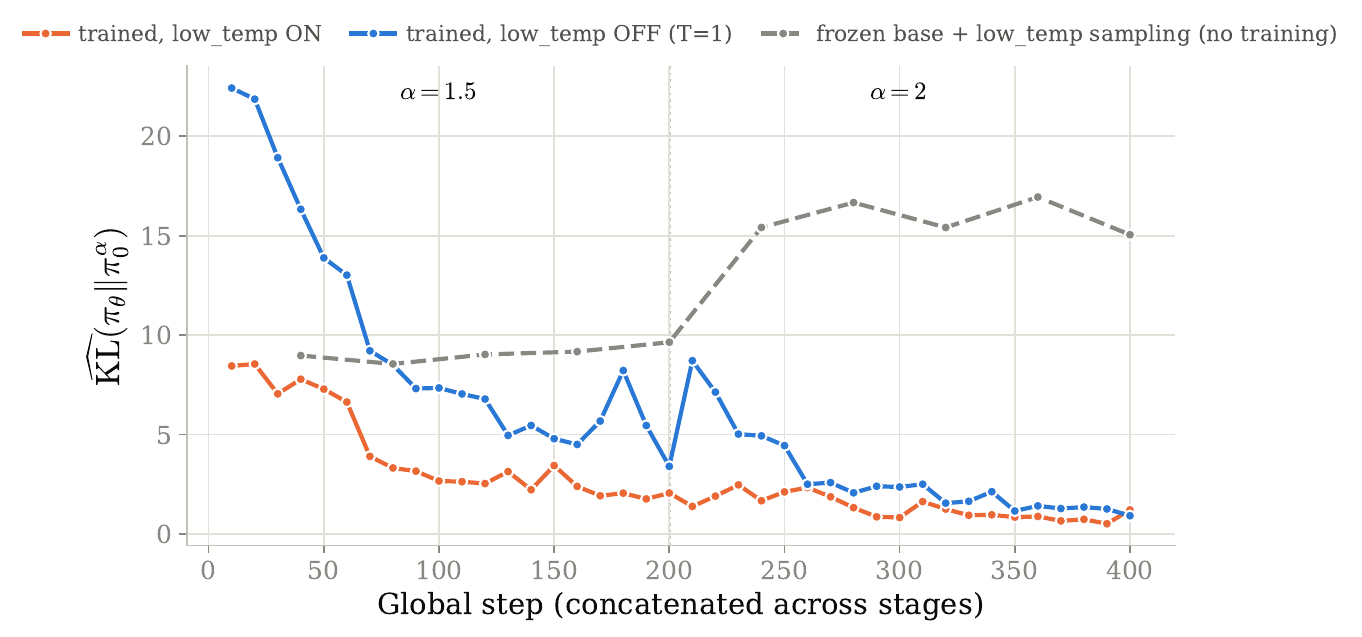}
   \caption{Progression of KL divergence to the target distribution over the course of training (lower is better). By maintaining a valid probability distribution, we are able to compute estimates of KL to track training stability. We observe faster and more stable convergence when using low-temperature sampling.}
    \label{fig:warm_start}
\end{figure}

\section{Implementation Details}
\label{app:implementation}

\paragraph{Model and PEFT.} We conduct experiments with three base models: \texttt{Qwen/Qwen2.5-7B-Instruct}, \texttt{Qwen/Qwen2.5-7B} (base), and \texttt{Phi-3.5-mini-instruct}, all loaded in bf16 with \texttt{sdpa} attention. LoRA is applied with rank 16, $\alpha=32$, dropout 0, no bias, targeting all attention and MLP projections (\texttt{q/k/v/o\_proj}, \texttt{gate/up/down\_proj}). The reference policy $\pi_0$ is the same model with the adapter disabled, avoiding a second resident copy of the base model. The LM head is cast to fp32 before PEFT wrapping via a custom pre-PEFT hook, since TRL's built-in cast breaks on PEFT-wrapped models.

\paragraph{Batching and sampling.} Each optimizer step uses 32 completions per question accumulated over 10 questions, for an effective batch of 320 rollouts per step; one micro-batch corresponds to one question's group, shuffled across the dataset. Prompts are capped at 768 tokens and completions at 1564 tokens. Generation uses a colocated vLLM engine (GPU memory utilization 0.15); truncated importance sampling correction is disabled, as $\pi_0$ rescoring under Hugging Face is treated as ground truth and the vLLM/HF discrepancy is attributable to bf16 kernel noise (verified at the start of training to be within 0.15 nats/token). No-grad old-policy and teacher log-probability passes use Apple's Cut Cross Entropy to avoid materializing the full $(B, T, V)$ logits tensor; the gradient-carrying forward pass inside the TRL loss remains on the standard path.

\paragraph{Optimization.} We use a learning rate of $1\mathrm{e}{-4}$ on LoRA parameters only, with no weight decay, Adam ($\beta_1=0.9$, $\beta_2=0.999$), and gradient-norm clipping at 1.0. The learning-rate schedule is per-stage rather than global: cosine decay to zero with a single restart per stage, 3\% warmup of that stage's steps, and a minimum LR ratio of 0; optimizer state is reset at each stage boundary along with a fresh warmup. We use \texttt{num\_iterations=1}, i.e.\ a purely on-policy REINFORCE estimator with no PPO-style replay or clipping.

\paragraph{Other.} A single adapter is trained across the entire $\alpha$ schedule, without rebuilding the trainer or vLLM engine between stages. All runs use seed 0.

\end{document}